%% file: cgnn.tex
\documentclass[letterpaper]{article}
\usepackage[utf8]{inputenc}
\usepackage{times, amsmath, amsfonts, booktabs,amsthm,stmaryrd}
\usepackage{ijcai18}
\usepackage{url, subcaption, microtype, graphicx, tikz}

\newtheorem{prop}{Proposition}
\usepackage[overload]{empheq}

\usetikzlibrary{arrows, calc, positioning, decorations.pathreplacing}

\tikzset{
    >=stealth',
    punkt/.style={
           circle,
           draw=black, thick,
           minimum height=1.75em,
           inner sep=0pt,
           text centered},
    pil/.style={
           ->,
           thick}
}

\newcommand{\Pa}[1]{\text{Pa}({#1}; \mathcal{G})}
\newcommand{\Ne}[1]{\text{Ne}({#1}; \mathcal{S})}

\makeatletter
\newcommand*{\indep}{%
  \mathbin{%
    \mathpalette{\@indep}{}%
  }%
}
\newcommand*{\nindep}{%
  \mathbin{
    \mathpalette{\@indep}{\not}
  }%
}
\newcommand*{\@indep}[2]{%
  \sbox0{$#1\perp\m@th$}
  \sbox2{$#1=$}
  \sbox4{$#1\vcenter{}$}
  \rlap{\copy0}
  \dimen@=\dimexpr\ht2-\ht4-.2pt\relax
  \kern\dimen@
  {#2}%
  \kern\dimen@
  \copy0 
} 
\makeatother

\begin{document}\thispagestyle{empty}
    \title{Causal Generative Neural Networks}


    \author{\textbf{%
            Olivier Goudet\thanks{Joint first author 
            (\texttt{firstname.lastname@lri.fr}). Rest of authors ordered alphabetically.} $\,^1$,
            Diviyan Kalainathan\footnotemark[1] $\,^1$,
            Philippe Caillou$^1$,
            Isabelle Guyon$^1$,
            }\\
            \textbf{%
            David Lopez-Paz$^2$,
            Mich\`ele Sebag$^1$}\\
            $^1$TAU, CNRS -- INRIA $-$ LRI, Univ. Paris-Sud, Univ.
            Paris-Saclay\\
            $^2$Facebook AI Research
            }
    
    \maketitle
    \input{01_introduction}
    \input{02_causal_modeling}

    \input{03_model}
    \input{04_experiments}
    \input{05_conclusion}

    \clearpage
    \newpage
    \bibliographystyle{named}
    \bibliography{cgnn}
\end{document}

%% file: 01_introduction.tex
\begin{abstract}
We present Causal Generative Neural Networks (CGNNs) to learn functional causal models from observational data. CGNNs leverage conditional independencies and distributional asymmetries to discover bivariate and multivariate causal structures. CGNNs make no assumption regarding the lack of confounders, and learn  a differentiable generative model of the data by using backpropagation. Extensive experiments show their good performances comparatively to the state of the art in observational causal discovery on both simulated and real data, with respect to cause-effect inference, v-structure identification, and multivariate causal discovery. 
\end{abstract}

\section{Introduction }

Deep learning models have shown tremendous predictive abilities in image classification, speech
recognition, language translation, game playing, and much more \cite{Goodfellow-et-al-2016}.
However, they often mistake
correlation for causation \cite{stock}, which can have catastrophic consequences for agents that plan and decide from observation.

The gold standard to discover causal relations is to perform experiments
\cite{pearl2003causality}. However, whenever experiments are expensive,
unethical, or impossible to realize, there is a
need for \emph{observational causal discovery}, that is, the estimation of
causal relations from observation alone \cite{spirtes2000causation,PetJanSch17}. In observational causal discovery, some authors exploit 
distributional asymmetries to discover bivariate causal relations
\cite{hoyer2009nonlinear,zhang2009identifiability,daniusis2012inferring,stegle2010probabilistic,lopez2015towards,fonollosa2016conditional}, while others rely on conditional independence to discover structures on three or more variables \cite{spirtes2000causation,chickering2002optimal}. Different approaches rely on different but equally strong assumptions, such as linearity \cite{shimizu2006linear}, additive noise \cite{zhang2009identifiability,peters2014causal}, determinism
\cite{daniusis2012inferring}, or a large corpus of annotated causal relations
\cite{lopez2015towards,fonollosa2016conditional}.  Among the most promising approaches are score-based methods \cite{chickering2002optimal}, assuming the existence of external \emph{score-functions} that must be powerful enough to detect diverse causal relations. 
Finally, most methods are not differentiable, thus unsuited for deep learning pipelines.
 
The ambition of \textbf{Causal Generative Neural Network (CGNNs)} is to provide a unified approach. CGNNs learn functional causal models (Section~\ref{sec:fcm}) as generative neural networks, trained by backpropagation to minimize the Maximum Mean Discrepancy (MMD) \cite{gretton2007kernel,li2015generative} between the observational and the generated data (Section~\ref{sec:cgnn}). Leveraging the representational power of deep generative models, CGNNs account for both distributional asymmetries and conditional independencies, tackle the bivariate and multivariate cases, and deal with hidden variables (confounders). They estimate both the causal graph underlying
the data and the full joint distribution, through the architecture and the weights of generative networks. Unlike previous 
approaches, CGNNs allow non-additive noise terms to model flexible conditional distributions. Lastly, they define differentiable joint distributions, which can be embedded within deep architectures.
Extensive experiments show the state-of-the-art performance of CGNNs (Section~\ref{sec:exps}) on cause-effect inference, v-structure identification, and multivariate causal discovery with hidden variables.\footnote{Code available at \url{https://github.com/GoudetOlivier/CGNN}. Datasets available at \url{http://dx.doi.org/10.7910/DVN/3757KX} and \url{http://dx.doi.org/10.7910/DVN/UZMB69}.}

%% file: 02_causal_modeling.tex
\section{The language of causality: FCMs}
\label{sec:fcm}

A Functional Causal Model (FCM) upon a random variable vector $X = (X_1, \ldots, X_d)$ is a triplet $C = (\mathcal{G}, f, \mathcal{E})$, representing a set of equations: 
\begin{equation}   {X}_i  \leftarrow {f}_i({X}_{\Pa{i}}, {E}_i), {E}_i \sim {\mathcal{E}}, \mbox{ for } i=1,\ldots, d
    \label{eq:1}
\end{equation}

Each equation characterizes the direct
causal relation from the set of causes $X_{\Pa{i}} \subset \{X_1, \ldots, X_d\}$ to observed variable $X_i$, described by some \emph{causal mechanism} $f_i$ up to the effects of noise variable $E_i$ drawn after distribution $\mathcal{E}$, accounting for all unobserved phenomenons. For simplicity, $X_i$ interchangeably denotes an observed variable and a node in graph $\mathcal{G}$. There exists a direct causal relation from $X_j$ to $X_i$, written $X_j \to X_i$, iff there exists a directed edge from $X_j$ to $X_i$ in $\mathcal{G}$. 
In the following, we restrict ourselves to considering Directed Acyclic Graph (DAG) $\mathcal{G}$ (Fig.~\ref{figure:causalnetwork}) and $\cal E$ is set to the uniform distribution on $[0,1]$, $U[0,1]$.

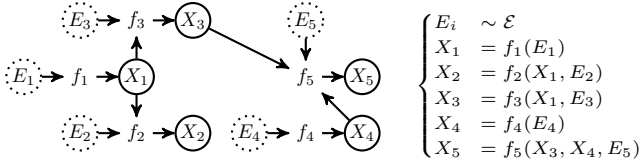
\begin{figure}[h]
    \begin{subfigure}{0.2\textwidth}
    \begin{center}
    \begin{tikzpicture}[scale=0.75, every node/.style={scale=0.75},node distance=2cm]
        \node[punkt, dotted] (e1) at (0,0) {$E_1$};
        \node[] (f1) at (1,0) {$f_1$};
        \node[punkt] (x1) at (2,0) {$X_1$};
        \draw[pil] (e1) -- (f1);
        \draw[pil] (f1) -- (x1);
        
        \node[punkt, dotted] (e2) at (1,-1) {$E_2$};
        \node[] (f2) at (2,-1) {$f_2$};
        \node[punkt] (x2) at (3,-1) {$X_2$};
        \draw[pil] (e2) -- (f2);
        \draw[pil] (f2) -- (x2);
        \draw[pil] (x1) -- (f2);
        
        \node[punkt, dotted] (e3) at (1,1) {$E_3$};
        \node[] (f3) at (2,1) {$f_3$};
        \node[punkt] (x3) at (3,1) {$X_3$};
        \draw[pil] (e3) -- (f3);
        \draw[pil] (f3) -- (x3);
        \draw[pil] (x1) -- (f3);
        
        \node[punkt, dotted] (e4) at (4,-1) {$E_4$};
        \node[] (f4) at (5,-1) {$f_4$};
        \node[punkt] (x4) at (6,-1) {$X_4$};
        \draw[pil] (e4) -- (f4);
        \draw[pil] (f4) -- (x4);
        
        \node[punkt, dotted] (e5) at (5,1) {$E_5$};
        \node[] (f5) at (5,0) {$f_5$};
        \node[punkt] (x5) at (6,0) {$X_5$};
        \draw[pil] (e5) -- (f5);
        \draw[pil] (f5) -- (x5);
        
        \draw[pil] (x3) -- (f5);
        \draw[pil] (x4) -- (f5);
    \end{tikzpicture}
    \end{center}
    \end{subfigure}
    \hfill
     \begin{subfigure}{0.18\textwidth}
	\scriptsize	
    \[
    \begin{cases}
        E_i &\sim \mathcal{E} \\
        X_1 &= f_1(E_1) \\
        X_2 &= f_2(X_1,E_2) \\
        X_3 &= f_3(X_1,E_3) \\
        X_4 &= f_4(E_4) \\
        X_5 &= f_5(X_3,X_4,E_5)
    \end{cases}\]
    \end{subfigure}
    \caption{Example FCM for $X =
    (X_1, \ldots, X_5)$.}
    \label{figure:causalnetwork}
\end{figure}

\subsection{Generative models and interventions} \label{sec:hathat}

The generative model associated to 
FCM $(\mathcal{G}, f, \mathcal{E})$ proceeds by first drawing  $e_i \sim \mathcal{E}$ for all $i=1,\ldots,d$, then in topological order of $\mathcal{G}$ computing $x_i = f_i(x_{\Pa{i}}, e_i)$.

Importantly, the FCM supports interventions, that is, freezing a variable $X_i$ to some constant $v_i$. The resulting joint distribution noted $P_{\text{do}(X_i = v_i)}(X)$, called  \emph{interventional distribution} \cite{pearl2009causality}, can be computed from the FCM by discarding all causal influences on $X_i$ and clamping its value to $v_i$.  It is emphasized that intervening
is different from conditioning (\emph{correlation does not imply causation}). The knowledge of interventional distributions is essential for e.g., public policy makers, wanting to estimate the overall effects of a decision on a given variable. 

\subsection{Formal background and notations}

In this section, we introduce notations and definitions and prove the representational power of FCMs.

Two random variables $(X, Y)$ are
\emph{conditionally independent} given $Z$ if $P(X, Y | Z) = P(X|Z) P(Y|Z)$.
Three random variables $(X, Y, Z)$ form a v-structure iff 
$X \to Z \leftarrow Y$.  The random variable $Z$ is a confounder
(or common cause) of the pair $(X, Y)$ if  $(X, Y, Z)$ have
causal structure $X \leftarrow Z \to Y$.  The skeleton $\mathcal{U}$ of a DAG
$\mathcal{G}$ is obtained by replacing all the directed edges in $\mathcal{G}$
by undirected edges.

Discovering the causal structure of a random vector is a difficult
task in all generality. For this reason, the literature
in causal inference relies on a set of common assumptions
\cite{pearl2003causality}. The \emph{causal sufficiency} assumption states
that there are no unobserved confounders. The \emph{causal Markov} assumption
states that all the d-separations in the causal graph $\mathcal{G}$ imply
conditional independences in the observational distribution $P$. The
\emph{causal faithfulness} assumption states that all the conditional
independences in the observational distribution $P$ imply d-separations in the
causal graph $\mathcal{G}$. A \emph{Markov equivalence class} denotes the set
of graphs with same set of d-separations. 



 \begin{prop}{\textbf{Representing joint distributions with FCMs}\label{prop1}}
 
 Let $X = (X_1, \ldots, X_d)$ denote a set of continuous random variables with joint distribution $P$, and further assume that the joint density function $h$ of $P$ is continuous and strictly positive on a compact subset of $\mathbb{R}^{d}$, and zero elsewhere. Letting $\cal G$ be a DAG such that $P$ can be factorized along $\cal G$, 
 $$ P(X) = \prod_i P(X_i | X_{\Pa{i}})$$
 there exists $f = (f_1, \ldots, f_d)$ with $f_i$ a continuous function with compact support in $\mathbb{R}^{|\Pa{i}|}\times [0,1]$ such that $P(X)$ equals the generative model defined from FCM $({\cal G}, f, {\cal E})$. 
\end{prop}
\begin{proof}
By induction on the topological order of $\cal G$, taking inspiration from \cite{carlier2016vector}. Let $X_i$ be such that $|\Pa{i}|=0$ and consider the cumulative distribution $F_i(x_i)$ defined over the domain of $X_i$ ($F_i(x_i) = Pr(X_i < x_i)$). $F_i$ is strictly monotonous as the joint density function is strictly positive therefore its inverse, the quantile function $Q_i: [0,1] \mapsto dom(X_i)$ is defined and continuous. By construction, $Q_i(e_i) =F_i^{-1}(x_i)$ and setting $Q_i = f_i$ yields the result.\\
Assume $f_i$ be defined for all variables $X_i$ with topological order less than $m$.  Let $X_j$ with topological order $m$ and $Z$ the vector of its parent variables. For any noise vector $e = (e_i, i \in \Pa{j})$ let $z = (x_i,  i \in \Pa{j})$ be the value vector of variables in $Z$ defined from $e$. The conditional cumulative distribution $F_j(x_j | Z=z) = Pr(X_j < x_j | Z=z)$ is strictly continuous and monotonous wrt $x_j$, and can be inverted using the same argument as above. Defining $f_j(z,e_j) = F_j^{-1}(z,x_j)$ yields the result. 
\end{proof}




%% file: 03_model.tex
\section{Causal generative neural networks}
\label{sec:cgnn}
Let $X = (X_1, \ldots, X_d)$ denote a set of continuous 
random variables with joint distribution $P$. Under same conditions as in Proposition 1, ($P(X)$ being decomposable to graph $\cal G$, with continuous and strictly positive joint density function on a compact in $\mathbb{R}^d$ and zero elsewhere), it is shown that there exists a generative neural network called \textbf{CGNN} (Causal Generative Neural Network), that approximates $P(X)$ with arbitrary accuracy. 

\subsection{Approximating continuous FCMs with CGNN} 
Firstly, given $\cal G$, it is shown that there exists a set of networks $\hat{f} = (\hat{f}_1, \ldots, \hat{f}_d)$  such that the generative model $\hat P$ defined by $\hat{X}_i = \hat{f}_i(\hat{X}_{\Pa{i}}, E_i)$ with $E_i \sim {\cal E}$ defines a joint distribution arbitrarily close to $P$. 

 \begin{prop}{\label{prop2}}
For $m \in [[1,d]]$, let $Z_m$ denote the set of variables with topological order less than $m$ and let $d_m$ be its size. For any $d_m$-dimensional vector of noise values $e^{(m)}$, let $z_m(e^{(m)})$ (resp. $\widehat{z_m}(e^{(m)})$) be the vector of values computed in topological order from $f$ (resp. $\hat f$). 
For any $\epsilon > 0$, there exists a set of networks $\hat{f}$ with architecture $\cal G$ such that 
\begin{equation}
\forall e^{(m)},  \|z_m(e^{(m)})- \widehat{z_m}(e^{(m)})\| < \epsilon
\label{eq:prop2}
\end{equation}
\end{prop}
\begin{proof}
By induction on the topological order of $\cal G$. Let $X_i$ be such that $|\Pa{i}|=0$. 
Following the universal approximation theorem \cite{cybenko1989approximation}, as $f_i$ is a continuous function over a compact of $\mathbb{R}$, there exists a neural net $\hat{f_{i}}$ such that $\|f_i - \hat{f_{i}}\|_\infty < \epsilon/d_1$. Thus  Eq. \ref{eq:prop2} holds for the set of networks $\hat{f_i}$ for $i$ ranging over variables with topological order 0.\\
Let us assume that Prop. 2 holds up to $m$, and let us assume for brevity that there exists a single variable $X_j$ with topological order $m +1$. Letting $\hat{f_j}$ be such that $\|f_j - \hat{f_j}\|_\infty < \epsilon/3$ (based on the universal approximation property), letting $\delta$ be such that for all $u$ $\|\hat f_j(u) - \hat f_j(u+\delta)\|< \epsilon/3$ (by absolute continuity) and letting $\hat f_i$ satisfying Eq. \ref{eq:prop2} for $i$ with topological order less than $m$ for $min(\epsilon/3,\delta)/d_{m}$, it comes: 
$\|(z_m,f_j(z_m,e_j)) - (\hat z_m,\hat{f_j}(\hat{z_m}, e_j))\| \le \|z_m - \hat z_m\| + |f_j(z_m,e_j) - \hat{f_j}(z_m, e_j)| +  | \hat{f_j}(z_m,e_j) - \hat{f_j}(\hat{z_m}, e_j)| < \epsilon/3 + \epsilon/3 + \epsilon/3$, which ends the proof.
\end{proof}

\subsection{Scoring metric \label{scoremetric}}

The architecture and the network weights are trained and optimized using a score-based approach \cite{chickering2002optimal}. The ideal score, to be minimized, is the distance between the joint distribution $P$ associated with the ground truth FCM, and the joint distribution $\widehat{P}$ defined by the estimated $(\hat{\cal G}, \hat{f}, {\cal E})$. A tractable approximation thereof is given by the Maximum Mean Discrepancy (MMD) \cite{gretton2007kernel} between the $n$-sample observational data $\cal D$, and an $n$- sample $\widehat{\cal D}$ sampled after $\widehat{P}$. Overall, $\widehat{C}$ is trained by minimizing   
\begin{equation}
  S(\widehat{\mathcal{G}}, \mathcal{D}) =
  - \widehat{\text{MMD}}_k(\mathcal{D}, \widehat{\mathcal{D}}) - \lambda 
    |\widehat{\mathcal{G}}|,
  \label{eq:the_loss}
\end{equation}
with $|\widehat{\mathcal{G}}|$ the number of edges in $\hat{\cal G}$ and $\widehat{MMD}_k$ defined as:
\begin{footnotesize}
\begin{equation*}
\frac{1}{n^2} \sum_{i, j = 1}^{n} k(x_i, x_j) +
\frac{1}{n^2} \sum_{i, j = 1}^{n} k(\hat{x}_i, \hat{x}_j)
- \frac{2}{n^2} \sum_{i,j = 1}^n k(x_i, \hat{x}_j),
\end{equation*}
\end{footnotesize}
\noindent where kernel $k$ usually is taken as the Gaussian kernel ($k(x,x') = \exp(-\gamma \|x-x'\|_2^2)$). The MMD statistic, with quadratic complexity in the sample size, has the good property that it is zero if and only if $P
= \hat{P}$ as $n$ goes to infinity \cite{gretton2007kernel}. For scalability, a linear approximation of the MMD statistics based on $m$ random features \cite{dlp},
called $\widehat{\text{MMD}}_k^m$, will also be used in the experiments. Due to the Gaussian kernel being differentiable, $\widehat{MMD}_k$ and $\widehat{\text{MMD}}_k^m$ are differentiable, and backpropagation can be used to learn the CGNN made of networks $\hat f_i$ structured along $\hat{\cal G}$. 

It is shown that the distribution $\hat P$ of the CGNN can estimate the true observational distribution of the (unknown) FCM up to an arbitrary precision, under the assumption of an infinite observational sample:

\begin{prop}{\label{prop3}}
Let  $\mathcal{D}$ be an infinite observational sample generated from $({\cal G}, f, {\cal E})$.
With same notations as in Prop. 2, for every $\epsilon >0$, there exists a set $\widehat{f_\epsilon} = (\hat f_1, \ldots \hat f_d)$ such that the MMD between $\cal D$ and an infinite size sample $\widehat{\cal D}_\ell$ generated from $({\cal G},\widehat{f_\epsilon},\cal E)$ is less than $\epsilon$.
 \end{prop}
 
\begin{proof}
According to Prop. \ref{prop2} and with same notations, letting $\epsilon_\ell > 0$ go to 0 as $\ell$ goes to infinity, consider  ${\hat f}_\ell=(\hat f^{\ell}_1 \ldots \hat f^{\ell}_d)$ and $\hat{z_\ell}$ defined from ${\hat f}_\ell$ such that for all $e \in [0,1]^d$, $\|z(e)- \widehat{z}_\ell(e)\| < \epsilon_\ell$. 

Let $\{ \hat{\mathcal{D}_\ell} \}$ denote the infinite sample generated after $\hat{f_\ell}$.
The score of the CGNN $(\mathcal{G},\hat{f_\ell},{\cal E})$ is $ \widehat{\text{MMD}}_k(\mathcal{D}, \hat{\mathcal{D}_\ell}) =  \mathbb{E}_{e,e'}[k(z(e),z(e')) - 2  k(z(e),\widehat{z}_\ell(e')) + k(\widehat{z}_\ell(e), \widehat{z}_\ell(e'))]$.


As $\hat{f_\ell}$ converges towards $f$ on the compact $[0,1]^d$, using the bounded convergence theorem on a compact subset of $\mathbb{R}^{d}$,  $\widehat{z_\ell}(e) \rightarrow z(e)$ uniformly for $\ell \rightarrow \infty$, it follows from the Gaussian kernel function being bounded and continuous that $\widehat{\text{MMD}}_k(\mathcal{D}, \hat{\mathcal{D}_\ell}) \rightarrow 0$, when $\ell \rightarrow \infty$.
\end{proof}


CGNN benefits  from i) the representational power of generative networks to exploit distributional asymmetries; ii) the overall approximation of the joint distribution of the observational data to  exploit conditional independences, to handle bivariate and multivariate causal modeling.

\subsection{Searching causal graphs with CGNNs}

The exhaustive exploration of all DAGs with $d$ variables is super-exponential in $d$, preventing the use of brute-force methods for observational causal discovery even for moderate $d$.
Following \cite{tsamardinos2006max,nandy2015high}, we assume known skeleton for $\mathcal{G}$, obtained via
domain knowledge or a feature selection algorithm 
\cite{yamada2014high} under standard assumptions such as causal Markov,
faithfulness, and sufficiency. 
Given a skeleton on $X$ and the regularized MMD score~\eqref{eq:the_loss}, CGNN follows a greedy procedure to find $\cal G$ and $f_i$:
\begin{itemize}
\item Orient each $X_i - X_j$ as $X_i \to X_j$ or $X_j \to X_i$ by selecting the associated 2-variable CGNN with best score.
\item Follow paths from a random set of nodes until all nodes are reached. Edges pointing towards a visited node reveal cycles, so must be reversed.
\item For a number of iterations, reverse the edge 
  that leads to the maximum improvement of the score $S({\mathcal{G}}, \mathcal{D})$ over a $d$-variable CGNN, without creating a cycle. 
  
 At the end of this process, we evaluate a confidence score for any edge $X_i \rightarrow X_j$ as: 
\begin{equation}
\label{eq:conf}
V_{X_i \rightarrow X_j} = S({\mathcal{G}}, \mathcal{D}) - S({\mathcal{G} - \{X_i \rightarrow X_j\}}, \mathcal{D}).
\end{equation}

\end{itemize}
\subsection{Dealing with hidden confounders} The search method above relies on the
causal sufficiency assumption (no
confounders).  We relax this assumption 
as follows.  Assuming confounders, each edge
$X_i - X_j$ in the skeleton is  due to one out of three
possibilities: either $X_i \to X_j$, $X_j \leftarrow X_i$, or there exists an
unobserved variable $E_{i,j}$ such that $X_i \leftarrow E_{i,j} \to X_j$.
Therefore, each equation in the FCM is extended to: $X_i \leftarrow
f_i(X_{\Pa{i}}, E_{i, \Ne{i}}, E_i)$, where $\Ne{i} \subset \{1, \ldots d\}$ is
the set of indices of the variables adjacent to $X_i$ in the skeleton.
Each $E_{i,j} \sim \mathcal{E}$ represents the hypothetical
unobserved common causes of $X_i$ and $X_j$. For instance, hiding $X_1$
from the FCM in Fig.~\ref{figure:causalnetwork} would require
considering a confounder $E_{2,3}$.  Finally, when considering hidden
confounders, the above third search step considers three possible mutations of the
graph: reverse, add, or remove an edge. In this case, 
$\lambda |\hat{\mathcal{G}}|$ promotes simple graphs.

%% file: 04_experiments.tex
\section{Experiments}\label{sec:exps}

CGNN is empirically validated and compared to the state of the art on observational causal discovery of i) cause-effect
relations (Section~\ref{sec:exps:two}); ii)  v-structures
(Section~\ref{sec:exps:three}); iii) multivariate causal structures with no confounders 
(Section~\ref{sec:exps:multi}); iv) multivariate causal structures when relaxing the no-confounder assumption (Section~\ref{sec:exps:confounded}).

\subsection{Experimental setting}
\begin{figure*}[ht]
    \begin{center}
    \begin{subfigure}{0.55\textwidth}
        \includegraphics[width=\textwidth]{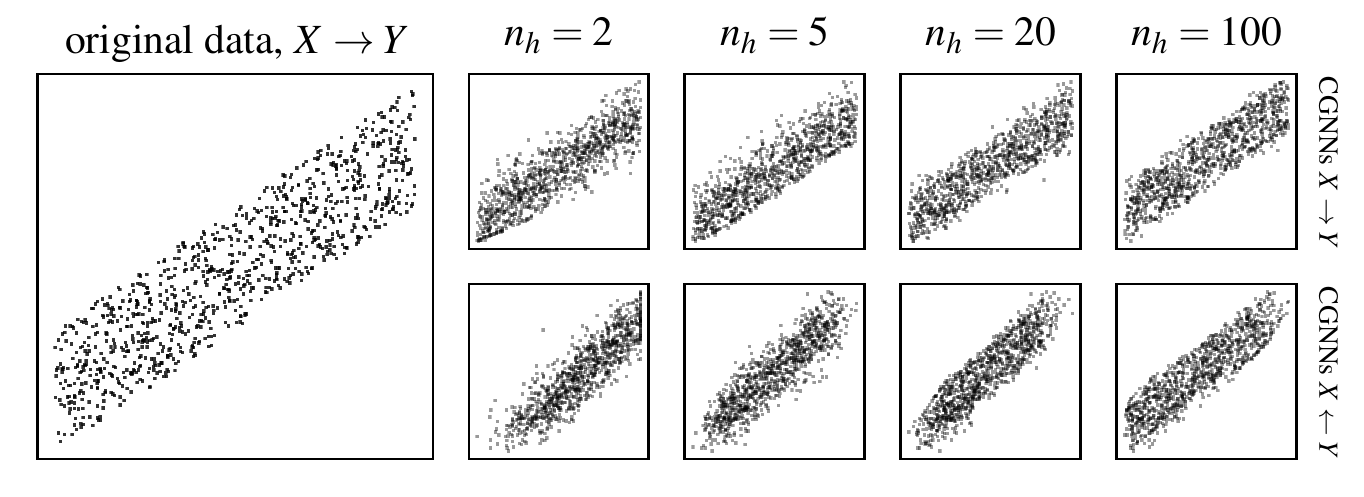}
        \caption{Samples.}
    \end{subfigure}
    \hskip 1cm
    \begin{subtable}{0.29\textwidth}
        \resizebox{\textwidth}{!}{
        \begin{tabular}{lllr}
            \toprule
            $n_h$ & $\widehat{\text{MMD}}_{X\rightarrow Y}$ & $\widehat{\text{MMD}}_{Y\rightarrow
            X}$ & Diff.\\
            \midrule
            2 & $32.0$ & $43.9$ &  \underline{$11.9$} \\
            5 & $29.6$ & $35.2$ &  \underline{$5.6 $} \\
            10 & $25.9$ & $32.5$ & \underline{$6.6$} \\
            20 & $25.7$ & $28.3$ & \underline{$2.6$} \\
            30 & $24.4$ & $26.8$ & \underline{$2.4$} \\
            40 &  $25.6$ & $25.6$ & $0.7$ \\
            50 &  $25.0$ & $25.0$ & $0.6$ \\
            100 & $24.9$ & $24.4$ & $-0.5$ \\
            \bottomrule
        \end{tabular}
        }
        \caption{Losses.}
    \end{subtable}
    \end{center}
    \caption{Samples and MMDs for CGNN models of different complexities
    (number of neurons) modeling the causal direction $X \to Y$ (top row) and the
      anticausal direction $X \leftarrow Y$ (bottom row) of a simple example.
      MMDs are averaged over $32$ runs, underlined numbers indicate statistical
    significance at $p=10^{-3}$.}
    \label{fig:para}
\end{figure*} 

MMD uses a sum of Gaussian kernels with bandwidths $\gamma \in \{0.005,
0.05,0.25,0.5, 1,5,50\}$. CGNN uses one-hidden-layer neural networks with $n_h$
ReLU units, trained with the Adam optimizer
\cite{2014arXiv1412.6980K} and initial learning rate of $0.01$, with full batch size $n=1500$. The generated data involved from noise variables are sampled anew
in each step. Each CGNN is trained for $n_{\text{train}} = 1000$ epochs and
evaluated on $n_{eval} = 500$ generated samples. Reported results are averaged over 32 runs for 
$\widehat{\text{MMD}}_k$ (resp. 64 runs for 
$\widehat{\text{MMD}}_k^m$). All experiments run on an Intel Xeon 2.7GHz CPU, and an NVIDIA GTX 1080Ti GPU. 

The most sensitive CGNN hyper-parameter is the number of hidden units $n_h$, governing the CGNN ability to model the causal mechanisms $f_i$: too small $n_h$, and data patterns may be missed; too large $n_h$, and overly complicated causal mechanisms might be retained. Overall, $n_h$ is problem-dependent, as illustrated on a toy problem  where two bivariate CGNNs are learned with $n_h = 2, 5, 20, 100$ (Fig.~\ref{fig:para}.a) from data generated by FCM: $X \sim \text{Uniform}[-2, 2], Y \leftarrow X + \text{Uniform}[0,
0.5]$.  Fig.~\ref{fig:para}.b shows the associated MMDs averaged on 32 independent runs,
and confirms the importance of cross-validating model capacity \cite{zhang2009identifiability}. 

\subsection{Discovering cause-effect relations}
\label{sec:exps:two}
Under the causal sufficiency assumption, the statistical dependence between two
random variables $X$ and $Y$ is either due to causal relation $X \to Y$ or $X \leftarrow Y$. The CGNN cause-effect accuracy is the fraction of edges in the graph skeleton that are rightly oriented, with Area Under Precision/Recall
curve (AUPR) as performance indicator. 

Five cause-effect inference datasets, covering a wide
range of associations, are used.  \emph{CE-Cha} contains 300 cause-effect pairs from the
challenge of \cite{guyon2013cepc}.  \emph{CE-Net} contains 300 artificial
cause-effect pairs generated using random distributions as causes, and neural
networks as causal mechanisms.
\emph{CE-Gauss} contains 300 artificial cause-effect pairs as generated by
\cite{mooij2016distinguishing}, using random mixtures of Gaussians as causes,
and Gaussian process priors as causal mechanisms.  \emph{CE-Multi} contains 300
artificial cause-effect pairs built with random linear and polynomial causal
mechanisms. In this dataset, we simulate additive or multiplicative noise,
applied before or after the causal mechanism.  \emph{CE-T\"ub} contains the 99
real-world scalar cause-effect pairs from the T\"ubingen dataset
\cite{mooij2016distinguishing}, concerning domains such as climatology,
finance, and medicine.  We set $n \leq 1500$. 

The baseline and competitor methods\footnote{\url{https://github.com/ssamot/causality}} include: i) the Additive Noise Model 
\emph{ANM} \cite{mooij2016distinguishing}, with Gaussian process regression
and HSIC independence test; ii)  the Linear Non-Gaussian Additive Model 
\emph{LiNGAM} \cite{shimizu2006linear}, a variant of Independent Component
Analysis to identify linear causal relations; iii) The Information Geometric Causal
Inference  \emph{IGCI} \cite{daniusis2012inferring}, with entropy estimator
and Gaussian reference measure; iv) the Post-Non-Linear model \emph{PNL}
\cite{zhang2009identifiability}, with HSIC test; v) The \emph{GPI} method
\cite{stegle2010probabilistic}, where the Gaussian process regression with
higher marginal likelihood is selected as causal direction; vi) the
Conditional Distribution Similarity statistic \emph{CDS}
\cite{fonollosa2016conditional}, which prefers the causal direction with
lowest variance of conditional distribution variances; vii)  the award-winning
method \emph{Jarfo} \cite{fonollosa2016conditional}, a random forest
classifier trained on the ChaLearn Cause-effect pairs and hand-crafted to
extract 150 features, including methods ANM, IGCI, CDS, and LiNGAM.
For each baseline and competitor method, a leave-one-dataset-out
scheme is used to select the best hyperparameters for each method (details omitted for brevity).  
\begin{table}[h!]
  \footnotesize
  \caption{Cause-effect relations: Area Under the Precision Recall curve on 5 benchmarks for the cause-effect experiments (weighted accuracy in parenthesis for T\"ub)}
  \label{table:pairwise}
  \centering
  \begin{tabular}{lccccc} 
    \toprule
    method & {Cha}  & {Net} & {Gauss} & {Multi} & {T\"ub }\\
    \midrule
    Best fit & 56.4 & 77.6 & 36.3 & 55.4 & 58.4 (44.9) \\
    LiNGAM & 54.3 & 43.7 & 66.5 & 59.3 & 39.7 (44.3) \\
    CDS  & 55.4 & 89.5 & 84.3 & 37.2 & 59.8 (65.5) \\
    IGCI  & 54.4 & 54.7 & 33.2 & 80.7 & 60.7 (62.6) \\
    ANM  & 66.3 & 85.1 & 88.9 & 35.5 & 53.7 (59.5) \\
    PNL  & 73.1 & 75.5 & 83.0 & 49.0 & 68.1 (66.2) \\
    Jarfo  & \underline{79.5} & \underline{92.7} & 85.3 & 94.6 & 54.5 (59.5)  \\
    GPI  & 67.4 & 88.4 & \underline{89.1} & 65.8 & 66.4 (62.6) \\
    \midrule
    \textbf{CGNN} ($\widehat{\text{MMD}}_k$) & 73.6 & 89.6 & 82.9 & \underline{96.6} & \underline{79.8} (74.4)  \\
    \textbf{CGNN} ($\widehat{\text{MMD}}^m_k$) & 76.5 & 87.0 & 88.3 & 94.2 & 76.9 (72.7) \\
    \bottomrule
  \end{tabular}
  \label{TableResultsPairwise}
\end{table}

As shown in Table~\ref{TableResultsPairwise},  i) linear regression methods are  dominated; ii) {CDS} and {IGCI} perform
well in some cases 
(e.g. when the entropy of causes is lower than those of effects); iii) {ANM} performs well when the additive noise assumption holds; iv) {PNL}, a generalization of {ANM},
compares favorably to the above methods; v)  {Jarfo}
performs well on artificial data
but badly on real examples. Lastly, generative methods {GPI} and \textbf{CGNN} ($\widehat{\text{MMD}}_k$) perform well on most datasets, including the real-world cause-effect pairs
{CE-T\"ub}, in counterpart for a higher computational cost (resp. 32 min on CPU for GPI and 24 min on GPU for CGNN). Using the linear MMD approximation \cite{dlp}, \textbf{CGNN} ($\widehat{\text{MMD}}^m_k$ as explained in Section \ref{scoremetric})  reduces the cost by a factor of 5 without hindering the performance. 
Overall, CGNN demonstrates competitive performance on the cause-effect inference problem, where it is necessary to discover distributional asymmetries.

\subsection{Discovering v-structures}
\label{sec:exps:three}
Considering random variables $(A, B, C)$ with skeleton
$A-B-C$, four causal structures are possible: the \emph{chain} $A \to B \to C$, the
\emph{reverse chain} $A \leftarrow B \leftarrow C$, the \emph{v-structure} $A
\to B \leftarrow C$, and the \emph{reverse v-structure} $A \leftarrow B \to C$.
Note that the chain, the reverse chain, and the reverse v-structure
are Markov equivalent, and therefore indistinguishable from each other using
statistics alone. This section thus examines the CGNN ability to identify v-structures.  

Let us consider an FCM with causal mechanisms $f_i = $ Identity
and Gaussian noise variables (e.g., $B \leftarrow A + E_B, E_B
\sim \mathcal{N}(0,1)$. As the joint distribution of one cause and its
effect is symmetrical, the bivariate methods used in the previous section do not apply
and the conditional independences among all three variables must be taken into account. 

The retained experimental setting trains a CGNN for every possible causal graph with skeleton $A-B-C$, and selects the one with minimal MMD. CGNN accurately discriminates the v-structures from the other ones $(0.202, 0.180)$, with a significantly lower MMD $(0.127)$ for the ground truth causal graph. This proof of concept shows the ability of CGNN to detect and exploit conditional independences among variables. 


\subsection{Discovering multivariate causal structures}
\begin{table*}[ht!]
  \caption{Average (std. dev.) results for the orientation of 20 artificial graphs given true skeleton (left), artificial graphs given skeleton with 20\% error (middle), and real protein network given true skeleton (right). $^*$ denotes statistical significance at $p=10^{-2}$.}
  \scriptsize
  \centering
  \begin{tabular}{l|ccc|ccc|ccc|}
    \toprule
     &  \multicolumn{3}{c}{Skeleton without error} & \multicolumn{3}{c}{Skeleton with 20\% of error} & \multicolumn{3}{c}{Causal protein network}\\
    method  & AUPR & SHD & SID & AUPR & SHD & SID & AUPR & SHD & SID \\
    \midrule
    \textit{Constraints}\\
    PC-Gauss & 0.67 (0.11) & 9.0 (3.4) & 131 (70) & 0.42 (0.06) & 21.8 (5.5) & 191.3 (73) & 0.19 (0.07) & 16.4 (1.3) & 91.9 (12.3)\\
    PC-HSIC & 0.80 (0.08) & 6.7 (3.2) & 80.1 (38) & 0.49 (0.06) & 19.8 (5.1) & 165.1 (67) & 0.18 (0.01) & 17.1 (1.1) & 90.8 (2.6) \\
    \midrule
    \textit{Pairwise}\\
    ANM & 0.67 (0.11) & 7.5 (3.0) & 135.4 (63) & 0.52 (0.10) & 19.2 (5.5) & 171.6  (66) & 0.34 (0.05) & 8.6 (1.3) & 85.9 (10.1)\\
    Jarfo & 0.74 (0.10)   & 8.1 (4.7) & 147.1 (94) & 0.58 (0.09) & 20.0 (6.8) & 184.8 (88) & 0.33 (0.02) & 10.2 (0.8) & 92.2 (5.2)\\
    \midrule
    \textit{Score-based}\\
        GES & 0.48 (0.13) & 14.1 (5.8) & 186.4 (86) & 0.37 (0.08) & 20.9 (5.5) & 209 (83) & 0.26 (0.01) & 12.1 (0.3) & 92.3 (5.4)\\ 
     LiNGAM & 0.65 (0.10) & 9.6 (3.8) & 171 (86) & 0.53 (0.10) & 20.9 (6.8) & 196 (83) & 0.29 (0.03) & 10.5 (0.8) & 83.1 (4.8) \\
     CAM & 0.69 (0.13)  & 7.0 (4.3) & 122 (76) & 0.51 (0.11) & \underline{15.6} (5.7) & 175 (80) & 0.37 (0.10) & 8.5 (2.2) & 78.1 (10.3)\\    
    \textbf{CGNN} ($\widehat{\text{MMD}}^m_k$) & 0.77 (0.09) &  7.1 (2.7) &  141 (59) & 0.54 (0.08) & 20 (10) & 179 (102) & 0.68 (0.07) & 5.7 (1.7) & 56.6 (10.0) \\
    \textbf{CGNN} ($\widehat{\text{MMD}}_k$) & \underline{0.89}* (0.09) & \underline{2.5}* (2.0) & \underline{50.45}* (45) & \underline{0.62} (0.12) & 16.9 (4.5) & \underline{134.0}* (55) & \underline{0.74}* (0.09)  & \underline{4.3}* (1.6) & \underline{46.6}* (12.4)\\
    \bottomrule
  \end{tabular}
  \label{table:multi}
\end{table*}
\label{sec:exps:multi}

Consider a random vector $X = (X_1,...,X_d)$. Our goal is to find the FCM of $X$ under the causal sufficiency assumption. At this point, we will assume known skeleton, so the problem reduces to orienting every edge. To that end, all experiments provide all algorithms {\em the
true graph skeleton}, so their ability to orient edges is compared in a fair way. This allows us to separate the task of orienting the graph from that of uncovering the skeleton.

\paragraph{Results on artificial data} 
We draw $500$ samples from $20$ training artificial causal graphs and $20$ test artificial causal graphs on 20 variables. Each variable has a number of parents uniformly drawn in $[[0,5]]$; $f_i$s are randomly generated polynomials involving additive/multiplicative noise.

We compare CGNN to the PC algorithm \cite{spirtes2000causation}, the score-based methods GES \cite{chickering2002optimal}, LiNGAM  \cite{shimizu2006linear}, causal additive model (CAM) \cite{peters2014causal}
and with the pairwise methods ANM and Jarfo. For PC, we employ the better-performing, order-independent version of the PC algorithm proposed by \cite{colombo2014order}. PC needs the specification of a conditional independence test. We compare PC-Gaussian, which
employs a Gaussian conditional independence test on Fisher z-transformations,
and PC-HSIC, which uses the HSIC conditional independence test with the Gamma approximation \cite{gretton2005kernel}. PC and GES are implemented in the \textit{pcalg} package \cite{kalisch2012causal}.

All hyperparameters are set on the training graphs in order to maximize the Area Under the Precision/Recall score (AUPR). For the Gaussian conditional independence test and the HSIC conditional independence test, the significance level achieving best result on the training set are respectively $0.1$ and $0.05$ .  For GES, the penalization parameter is set to $3$ on the training set.  For CGNN, $n_h$ is set to 20 on the training set. For CAM, the cutoff value is set to $0.001$.

Table \ref{table:multi} (left) displays the performance of all algorithms obtained by starting from the exact skeleton on the test set of artificial graphs and  measured from the AUPR (Area Under the Precision/Recall curve), the Structural Hamming Distance (SHD, the number of edge modifications to transform one graph into another)
and the Structural Intervention Distance (SID, the number of equivalent two-variable interventions between two graphs)
\cite{peters2013structural}.

CGNN obtains significant better results with SHD and SID  compared to the other algorithms when the task is to discover the causal from the true skeleton. Constraints based method PC with powerful HSIC conditional independence test is the second best performing method. It highlights the fact that when the skeleton is known, exploiting the structure of the graph leads to good results compared to pairwise methods using only local information. However CGNN and PC-HSIC are the most computationally expensive methods, taking an average of 4 hours on GPU and 15 hours on CPU, respectively.

The robustness of the approach is validated by randomly perturbing 20\% edges in the graph skeletons provided to all algorithms (introducing about 10 false edges over 50 in each skeleton). As shown on Table \ref{table:multi} (middle), and as could be expected, the scores of all algorithms are lower when spurious edges are introduced. Among the least robust methods are constraint-based methods; a tentative explanation is that they heavily rely on the graph structure to orient edges. By comparison pairwise methods are more robust because each edge is oriented separately. As CGNN leverages conditional independence but also distributional asymmetry like pairwise methods, it obtains overall more robust results when there are errors in the skeleton compared to PC-HSIC. 

CGNN obtains overall good results on these artificial datasets. It offers the advantage to  deliver a full generative model useful for simulation (while e.g., Jarfo and PC-HSIC only give the causality graph). 
 To explore the scalability of the approach, 5 artificial graphs with $100$ variables have been considered, achieving an AUPRC of $85.5 \pm 4$, in 30 hours of computation on four NVIDIA 1080Ti GPUs.  

\paragraph{Results on real-world data}
CGNN is applied to the protein network problem \cite{sachs2005causal}, using the Anti-CD3/CD28 dataset with  853 observational data points corresponding to general perturbations without specific interventions. All algorithms were given the skeleton of the causal graph \cite[Fig. 2]{sachs2005causal} with same hyper-parameters as in the previous subsection. We run each algorithm on 10-fold cross-validation. Table~\ref{table:multi} (right) reports average (std. dev.) results.

 Constraint-based algorithms obtain surprisingly low scores, because they cannot identify many V-structures in this graph. We confirm this by evaluating conditional independence tests for the adjacent tuples of nodes \textit{pip3}-\textit{akt}-\textit{pka}, \textit{pka}-\textit{pmek}-\textit{pkc}, \textit{pka}-\textit{raf}-\textit{pkc} and we do not find strong evidences for V-structure. Therefore methods based on distributional asymmetry between cause and effect seem better suited to this dataset. CGNN obtains good results compared to the other algorithms. Notably, Figure \ref{CGNN_cyto} shows that CGNN is able to recover the strong signal transduction pathway \textit{raf}$\rightarrow$\textit{mek}$\rightarrow$\textit{erk} reported in \cite{sachs2005causal} and corresponding to clear direct enzyme-substrate causal effect. CGNN gives important scores for edges with good orientation (solid line), and low scores (thinnest edges) to the wrong edges (dashed line), suggesting that false causal discoveries may be controlled by using the confidence scores defined in Eq. \eqref{eq:conf}.
  
\begin{figure}
    \centering

        \begin{tikzpicture}
          \def \h {11}
          \def \radius {1.8cm}
          \def \radiuscircle {0.2cm}
          \fontsize{6.0pt}{6.5pt}\selectfont
          \node[draw, circle,minimum size=\radiuscircle](praf) at ({360/\h * (1 - 1)}:\radius) {$raf$};
          \node[draw, circle,minimum size=\radiuscircle](pmek) at ({360/\h * (2 - 1)}:\radius) {$mek$};
          \node[draw, circle,minimum size=\radiuscircle](plcg) at ({360/\h * (3 - 1)}:\radius) {$plcg$};
          \node[draw, circle,minimum size=\radiuscircle](PIP2) at ({360/\h * (4 - 1)}:\radius) {$pip_2$};
          \node[draw, circle,minimum size=\radiuscircle](PIP3) at ({360/\h * (5 - 1)}:\radius) {$pip_3$};
          \node[draw, circle,minimum size=\radiuscircle](p44/42) at ({360/\h * (6 - 1)}:\radius) {$erk$};
          \node[draw, circle,minimum size=\radiuscircle](pakts473) at ({360/\h * (7 - 1)}:\radius) {$akt$};
          \node[draw, circle,minimum size=\radiuscircle](PKA) at ({360/\h * (8 - 1)}:\radius) {$pka$};
          \node[draw, circle,minimum size=\radiuscircle](PKC) at ({360/\h * (9 - 1)}:\radius) {$pkc$};
          \node[draw, circle,minimum size=\radiuscircle](P38) at ({360/\h * (10 - 1)}:\radius) {$p38$};
          \node[draw, circle,minimum size=\radiuscircle](pjnk) at ({360/\h * (11 - 1)}:\radius) {$jnk$};
          \draw[dash dot,pil, line width = 0.30pt] (praf) -- (PKA);
          \draw[pil, line width = 0.32pt] (PKA) -- (P38);
          \draw[pil, line width = 0.33pt] (PKC) -- (praf);
          \draw[dash dot,pil, line width = 0.35pt] (pjnk) -- (PKA);
          \draw[dash dot,pil, line width = 0.36pt] (PKC) -- (plcg);
          \draw[,pil, line width = 0.36pt] (PIP3) -- (plcg);
          \draw[pil, line width = 0.38pt] (PKC) -- (P38);
          \draw[dash dot,pil, line width = 0.40pt] (PIP2) -- (plcg);
          \draw[pil, line width = 0.43pt] (PKC) -- (pjnk);
          \draw[pil, line width = 0.44pt] (PIP2) -- (PKC);
          \draw[pil, line width = 0.48pt] (praf) -- (pmek);
          \draw[pil, line width = 0.52pt] (PKC) -- (pmek);
          \draw[pil, line width = 0.62pt] (PKA) -- (pmek);
          \draw[pil, line width = 0.89pt] (PIP2) -- (PIP3);
          \draw[pil, line width = 1.38pt] (PIP3) -- (pakts473);
          \draw[pil, line width = 1.47pt] (pmek) -- (p44/42);
          \draw[pil, line width = 1.55pt] (PKA) -- (pakts473);
          \draw[pil, line width = 1.60pt] (PKA) -- (p44/42);
         \end{tikzpicture}

  \caption{Causal protein network obtained with CGNN} 
\label{CGNN_cyto}

\end{figure}
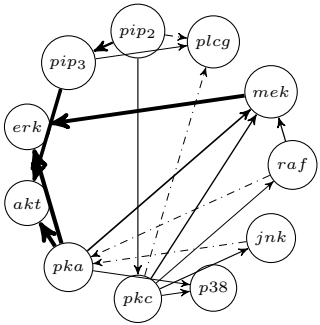

\subsection{Dealing with hidden confounders}
\label{sec:exps:confounded}

As real data often includes unobserved confounding variables, the robustness of CGNN is assessed by considering the previous artificial datasets while hiding some of the 20
observed variables in the graph. Specifically three random variables that cause at least two others in the same graph are hidden.  Consequently, the skeleton now includes additional edges $X-Y$ for all pairs of variables $(X,Y)$ that are consequences of the same hidden cause (confounder).  The goal in
this section is to orient the edges due to direct causal relations, and to remove those due to confounders.

We compare CGNN to the RFCI algorithm (Gaussian or HSIC conditional independence tests) \cite{colombo2012learning},
which is a modification of the PC algorithm that accounts for hidden variables. For CGNN, we set the hyperparameter $\lambda = 5 \times 10^{-5}$ fitted on the training graph dataset. Table~\ref{table:confounder} shows that CGNN is robust to confounders. Interestingly, true causal edges have high confidence, while edges due to confounding effects are removed or have low confidence. 

\begin{table}[h!]
\scriptsize
  \caption{AUPR, SHD and SID on causal discovery with confounders. $^*$ denotes significance at $p=10^{-2}$.}
  \centering
  \begin{tabular}{lcccr}
    \toprule
    method & AUPR &  SHD & SID \\
    \midrule
    RFCI-Gaussian & 0.22 (0.08) & 21.9 (7.5) & 174.9 (58.2) \\
    RFCI-HSIC & 0.41 (0.09) & 17.1 (6.2) & 124.6 (52.3) \\
    Jarfo & 0.54 (0.21) & 20.1 (14.8) & 98.2 (49.6) \\
    \midrule
    \textbf{CGNN} ($\widehat{\text{MMD}}_k$) & \underline{0.71}* (0.13) & \underline{11.7}* (5.5) & \underline{53.55}* (48.1)\\
    \bottomrule
  \end{tabular}
  \label{table:confounder}
  \vspace{-.5cm}
\end{table}

%% file: 05_conclusion.tex
\section{Conclusion}
\label{sec:future}

We introduced CGNN, a new framework to learn functional causal models from observational data based on generative neural networks. CGNNs minimize the maximum mean discrepancy between their generated samples and the observed data.
CGNNs combines the power of deep learning and the interpretability of causal models. Once trained, CGNNs are causal models of the world able to simulate the outcome of interventions.
Future work includes i) extending the proposed approach to categorical and temporal data, ii) characterizing sufficient identifiability conditions for the approach, and iii) improving the computational efficiency  of CGNN.